\documentclass[conference]{IEEEtran}
\IEEEoverridecommandlockouts
\usepackage{cite}
\usepackage{amsmath,amssymb,amsfonts,amsthm}
\usepackage{graphicx}
\usepackage{textcomp}
\usepackage{xcolor}
\usepackage{algorithm}
\usepackage{algpseudocode}
\usepackage{cleveref}

\def\BibTeX{{\rm B\kern-.05em{\sc i\kern-.025em b}\kern-.08em
    T\kern-.1667em\lower.7ex\hbox{E}\kern-.125emX}}

\newcommand{\bx}{\mathbf{x}}

\newcommand{\cM}{{\cal M}}
\newcommand{\cS}{{\cal S}}
\newcommand{\cA}{{\cal A}}
\newcommand{\R}{\mathbb{R}}

\newcommand{\ceil}[1]{\left \lceil {#1} \right \rceil }
\algnewcommand{\LeftComment}[1]{\Statex \(\triangleright\) #1}
\def\NoNumber#1{{\def\alglinenumber##1{}\State #1}\addtocounter{ALG@line}{-1}}

\newtheorem{theorem}{Theorem}
\newtheorem{lemma}{Lemma}

\newtheorem{definition}{Definition}

\begin{document}

\title{Secure Best Arm Identification\\ in the Presence of a Copycat}

\author{
Asaf Cohen\IEEEauthorrefmark{1} and Onur G{\"u}nl{\"u}\IEEEauthorrefmark{2}\IEEEauthorrefmark{3}
\vspace{0.8em}\\  % fine‑tunes spacing between names and affiliations
\IEEEauthorblockA{\IEEEauthorrefmark{1}\textit{School of Electrical and Computer Engineering}, Ben‑Gurion University of the Negev, \texttt{coasaf@bgu.ac.il}}
\vspace{0.4em}
\IEEEauthorblockA{\IEEEauthorrefmark{2}\textit{Information Coding Division}, Link{\"o}ping University, \texttt{onur.gunlu@liu.se}}  \IEEEauthorblockA{\IEEEauthorrefmark{3}\textit{Lehrstuhl f{\"u}r Nachrichtentechnik}, Technische Universi{\"a}t Dortmund}%
%\IEEEauthorblockN{Asaf Cohen}
%\IEEEauthorblockA{\textit{School of Electrical and Computer Engineering}\\
%Ben-Gurion University of the Negev \\
%coasaf@bgu.ac.il}
%\and
%\IEEEauthorblockN{Onur G{\"u}nl{\"u}}
%\IEEEauthorblockA{\textit{Information Coding Division} %\\
%Link{\"o}ping University\\
%onur.gunlu@liu.se}
}

\maketitle
%%%%%%%%%%%%%%%%%%%%%%%%%%%%%%%%%%%%%%%%%%
%%%%%%%%%%%%%%%%%%%%%%%%%%%%%%%%%%%%%%%%%%
\begin{abstract}
Consider the problem of best arm identification with a security constraint. Specifically, assume a setup of stochastic linear bandits with $K$ arms of dimension $d$. In each arm pull, the player receives a reward that is the sum of the dot product of the arm with an unknown parameter vector and independent noise. The player's goal is to identify the best arm after $T$ arm pulls. Moreover, assume a copycat Chloe is observing the arm pulls. The player wishes to keep Chloe ignorant of the best arm. 

While a minimax--optimal algorithm identifies the best arm with an $\Omega\left(\frac{T}{\log(d)}\right)$ error exponent, it easily reveals its best-arm estimate to an outside observer, as the best arms are played more frequently. A naive secure algorithm that plays all arms equally results in an $\Omega\left(\frac{T}{d}\right)$ exponent. In this paper, we propose a secure algorithm that plays with \emph{coded arms}. The algorithm does not require any key or cryptographic primitives, yet achieves an $\Omega\left(\frac{T}{\log^2(d)}\right)$ exponent while revealing almost no information on the best arm.
\end{abstract}
%%%%%%%%%%%%%%%%%%%%%%%%%%%%%%%%%%%%%%%%%%
%%%%%%%%%%%%%%%%%%%%%%%%%%%%%%%%%%%%%%%%%%
%%%%%%%%%%%%%%%%%%%%%%%%%%%%%%%%%%%%%%%%%%
%%%%%%%%%%%%%%%%%%%%%%%%%%%%%%%%%%%%%%%%%%
\section{Introduction}
Reinforcement learning is a framework that involves a user interacting with an unknown environment. During this interaction, the user takes actions and receives rewards based on those actions. The primary objective is to learn the optimal actions that maximize overall rewards. Multi-Armed Bandit (MAB) problems represent a specific case within this framework, where the user is faced with a selection of arms to pull. Each arm is associated with its own unique reward (or reward distribution), and the challenge is to pull the arms sequentially to maximize the cumulative rewards. MAB problems have diverse applications across various fields, including clinical trials \cite{villar2015multi}, finance (such as portfolio selection \cite{shen2015portfolio} and dynamic pricing \cite{misra2019dynamic}), recommendation systems \cite{silva2022multi}, influence maximization in social networks \cite{alshahrani2019influence}, and even routing \cite{tabei2023design}.

Best arm identification is a specific problem in which the player does not accumulate rewards. Instead, the player has the opportunity to pull $T$ arms before making a decision regarding the best arm. Performance is evaluated based on the error probability in identifying the best arm. To illustrate a specific application, consider \emph{dynamic pricing}. A seller preparing to launch a new product must assess its acceptance based on several characteristics. However, the seller is uncertain about the degree to which consumers value each characteristic. To address this, the seller tests a few initial product variations, each with different attribute values or weights influencing the overall appearance of the product. After experimenting with these initial variations, the seller must ultimately decide which combination of attribute weights yields the highest reward.

The aforementioned application raises a critical concern: an outside observer can see the initial product variations, and as the seller refines their tests, the observer can infer the market demand as well. An astute seller would want to conduct initial product tests and refine them to gain a better understanding of market demand, while minimizing the amount of information revealed to competing sellers. This raises the problem of \emph{secure best arm identification}.

\subsection{Main Contributions}
In this work, we focus on best arm identification in the linear--stochastic setting, were there are $K$ arms of dimension $d$. We present a novel model, with a novel solution, in which a secure algorithm effectively balances the need for accurate identification of the best arm while minimizing the information disclosed to external observers. The state--of--the--art best arm identification algorithm in \cite{yang2022minimax}, proposing a minimax-optimal algorithm, achieves an error exponent of $\Omega\left(\frac{T}{\log(d)}\right)$ but does so by exposing its best arm estimates, making them vulnerable to competitive analysis. Conversely, a naive secure approach that plays all arms equally results in a significantly worse error exponent of $\Omega\left(\frac{T}{d}\right)$. Note that contemporary learning applications consider datasets where the dimension is of the same order as the number of samples (arm pulls), hence $\Omega\left(\frac{T}{d}\right)$ may reflect poor performance. 

Our proposed algorithm achieves an error exponent of $\Omega\left(\frac{T}{\log^2(d)}\right)$ without relying on cryptographic keys or complex primitives. By employing a \emph{coding strategy}, the algorithm allows efficient exploration of arm variations while hiding valuable information about the best arm from competitors. In fact, the best arm declared by the algorithm is among $\frac{d}{2}$ of the arms it uses with the same frequency, resulting in $\log(d)-1$ bits of information missing for an outside observer. We argue that $\log(d)$ is the optimal equivocation, as the critical parameter for the estimation is $d$ rather than $K$. 

Revisiting the dynamic pricing example as a test case, we allow the seller to introduce ``coded products". In this approach, the seller can manipulate the weights assigned to each characteristic in a way that does not necessarily reflect their best current estimate. However, through this method, the seller is \emph{able to infer} the optimal weights, possibly by intelligently analyzing the rewards received from multiple previous initial product variations. The coded products thus help the seller gain information, yet they mix initial products such that an outside observer cannot understand which initial products are suspected to be more successful by the seller.   

Such a compromise between exploration and information security can mark a significant advancement in adaptive decision--making frameworks, yielding a practical tool for sellers and decision-makers in environments where confidentiality is paramount.

\subsection{Related Work}
A model of particular interest to this work is that of linear stochastic bandits \cite{abbasi2011improved}, and especially  \cite{yang2022minimax}, focusing on minimax optimal fixed-budget best arm identification.  \cite{yang2024best} introduces an additional constraint of minimal regret. Notably, the lower bound established in \cite{yang2022minimax} builds upon the non--linear, more general problem addressed in \cite{carpentier2016tight}, which is integral to understanding the limits of current approaches. 

The exploration of constrained stochastic bandit problems, such as those discussed in \cite{pacchiano2021stochastic}, is relevant even though they do not directly address security considerations. For instance, \cite{mayekar2023communication} examines best arm identification under communication constraints, establishing key connections between communication and decision--making. Moreover, \cite{kota2023almost} portrays a network of clients sharing a common set of arms, communicating with a central server over cost--incurring links, allowing the server to estimate the global best arm. The integration of meta--learning and emerging concepts, such as learning--to--learn, has begun to shape the MAB literature, as illustrated by \cite{cella2020meta}. 

% secure MAB
In terms of secure versions of the problem, there are a few recent works. Specifically, \cite{mishra2014private} considers private stochastic MAB, but the focus is on protecting the privacy of the data returned from the bandits, i.e., rewards are protected rather than the information gained by the learner. \emph{Note that even if rewards are protected or completely unobserved, the learner itself is not necessarily safeguarded. The actions taken by the learner during interactions with the environment can potentially disclose its inferred model.} In the same context, a differentially private stopping rule for MAB is discussed in \cite{sajed2019optimal}. The authors in \cite{jun2018adversarial} consider adversarial attacks, which are cases where an attacker can manipulate the rewards received by the learner to distract the learner or cause the learning process to reach a wrong conclusion. \cite{amani2019linear} considers a different model, in which the learner has to avoid unsafe actions that are unknown in advance. Thus, the learning algorithm has to be more conservative and explore possible actions safely. A safe learning algorithm is also given in \cite{khezeli2020safe}, where, again, the goal of the learner is to avoid actions which are ``unsafe", in this case, causing a reward that is below a safe threshold. \cite{ciucanu2019secure,ciucanu2022secure} consider the problem of best arm identification when outsourcing the data and computations to a honest--but--curious cloud. Their goals are that no cloud node can learn information about both the rewards and the arm ranking, and that the messages between the cloud nodes do not leak any information. However, they propose a cryptographic solution, focusing on complex homomorphic encryption. A similar treatment of the problem under the cumulative reward paradigm is given in \cite{ciucanu2020secure}. Moreover, \cite{chang2023distributed} considers stochastic bandits with corrupted and defective input commands, while \cite{chang2022covert} considers covert best arm identification.
%%%%%%%%%%%%%%%%%%%%%%%%%%%%%%%%%%%%%%%%%%
%%%%%%%%%%%%%%%%%%%%%%%%%%%%%%%%%%%%%%%%%%
\section{Problem Statement}
Assume $\log$ is of base $2$. For the dot product of two vectors $a$ and $b$ in $\R^d$, we use either $a^T b$ or $\langle a,b \rangle$.
\subsection{Best Arm Identification for Stochastic Linear Bandits}
Let $\cA = \{a(1),\ldots,a(K)\} \subset \R^d$ denote a finite set of arms. W.l.o.g., we assume the arms are unique, and $K > d$. Let $[T] = \{1,\ldots, T\}$ denote a finite set of trails (games). In the best arm identification problem, at each time $t \in [T]$, a learner (Leah) picks an arm $a_t$ from the set $\cA$, and plays it. She receives a result $X_t = a_t^T \theta^* + \eta_t$, where $\theta^* \in \R^d$ is an unknown parameter to Leah and $\eta_t$ is a noise term (e.g., i.i.d.\ sub--Gaussian random variables, each having a variance $1$). After $T$ games, Leah's goal is to identify the best arms $a(i^*)$, where we have $i^* = \arg\max_{1 \leq i \leq K} a(i)^T \theta^*$.

Leah's figure of merit is her error probability. That is, let $\hat{i}$ be Leah's estimate of the best arm after $T$ rounds. Leah wishes to minimize $P_e = P(\hat{i} \ne i^*)$.
The error probability $P_e$ will usually decrease exponentially with $T$ and we are interested in the best possible exponential decrease. Note that several algorithms for best arm identification use \emph{rounds} of arm--pulls. In such cases, we are interested in the total number of arm--pulls. E.g., \cite{yang2022minimax} uses roughly $\log d$ rounds, with $m \approx \frac{T}{\log d}$ arm--pulls in each. Furthermore, as we are interested in the exponential behavior of the error probability, we assume a large $T$, especially, $T >> d^2$.
%%%%%%%%%%%%%%%%%%%%%%%%%%%%%%%%%%%%%%%%%%
\subsection{Main Estimation Results}
The current literature on both linear MAB problems and various related problems ultimately converges to a fundamental estimation result. Given the actions $\{a_t\}_{t=1}^{T}$ (arms played in the linear MAB problem) and results $X_t = \langle a_t, \theta^* \rangle + \eta_t$, the least squares estimate of $\theta^*$ (with a regularization $\lambda$) is expressed as $\hat{\theta} = V(\lambda)^{-1} \sum_{t=1}^{T} X_t a_t$,
where we have $V(\lambda) = \lambda I + \sum_{t=1}^{T} a_t a_t^T$ (\cite{lattimore2020bandit}). The main direct result is that for any $w \in \R^d$, we have
\begin{equation}\label{upper bound for estimation}
P\left(\langle \hat{\theta}-\theta^*, w \rangle \ge \sqrt{2 \|w\|^2_{V(\lambda)^{-1}}\log\left(\frac{1}{\delta}\right)}  \right) \leq \delta 
\end{equation}
where we define $\|w\|^2_{V(\lambda)^{-1}} = w^T V(\lambda)^{-1} w$. Such a result is used in \cite{abbasi2011improved,yang2022minimax} and follow-up works to give direct estimation results for the suggested algorithms. 
%%%%%%%%%%%%%%%%%%%%%%%%%%%%%%%%%%%%%%%%%%
\subsection{Main Result for (un--secure) Best Arm Identification}
In \cite{yang2022minimax}, Yang and Tan propose an algorithm with $\log d$ rounds (assuming $d$ is a power of $2$ for simplicity). In the first round, at most $\frac{d(d+1)}{2}$ arms are tested. Each arm is pulled multiple times, according to a G--optimal distribution \cite{kiefer1960equivalence,pukelsheim2006optimal,damla2008linear}, focused on optimizing the estimation performance of $\theta^*$ (that is, regardless of the player's estimate of the index $i^*$). Starting from the second round, only the best $\frac{d}{2^{r-1}}$ arms are tested. Denoting by $p(i) = a(i)^T\theta^*$ the expected reward of arm $a(i)$, and assuming, w.l.o.g., that arms are ordered by expected rewards (i.e., $i^* = 1$), \cite{yang2022minimax} further defines $\Delta_i = p(1)-p(i)$ and $\text{H}_{2,\text{lin}} = \max_{2 \leq i \leq d} \frac{i}{\Delta_i^2}$ as a hardness parameter governed by the set $\cA$. $\text{H}_{2,\text{lin}}$ essentially captures how close are the arms to each other (via $\Delta_i$), hence how hard they are to distinguish. The resulting error probability is
\begin{equation*}
P(\hat{i} \ne i^*) \leq \exp\left(-\Omega\left(\frac{T}{H_{2,\text{lin}} \log d}\right)\right)
\end{equation*}
which is proved to be minimax optimal. However, as the algorithm in \cite{yang2022minimax} uses rounds, with the number of arms in each round decreasing by $1/2$, focusing on the best arms each time, it is clear that observing the arm pulls reveals information on the arms estimated as best. 
%%%%%%%%%%%%%%%%%%%%%%%%%%%%%%%%%%%%%%%%%%
%%%%%%%%%%%%%%%%%%%%%%%%%%%%%%%%%%%%%%%%%%
\section{Learning in the Presence of a Copycat}
Suppose now a \emph{copycat} Chloe observes Leah's actions (arm plays, $\{a_t\}_{t=1}^T$) without observing the results $\{X_t\}_{t=1}^T$. Leah wishes to identify the best arm with a low error probability, yet without revealing $i^*$ to Chloe. Note that Chloe \emph{can} gather information on $i^*$ without observing the results. For example, many elimination--based best arm identification algorithms use rounds of arm--pulls, and use $\hat{i}$, together with a single other arm, in the last round. Hence, Chloe by simply tossing a fair coin on the two arms in the last round can identify $i^*$ with an error probability 
\begin{eqnarray*}
    P_e^{Chloe} &=& P(\text{head} , \hat{i} = i^*) +P(\hat{i} \ne i^*) 
    \\
    &\leq& \frac{1}{2} + \exp\left(-\Omega\left(\frac{T}{H_{2,\text{lin}} \log d}\right)\right).
\end{eqnarray*}
In fact, Chloe has a distribution of $\left(\frac{1}{2},\frac{1}{2}\right)$ on a set of size $2$, which includes $a(i^*)$ with very high probability, hence, is \emph{essentially missing only one bit} of information in identifying the best arm. 

We thus define the security constraint as follows: we assume Chloe, after observing all arm plays $\{a_t\}_{t=1}^T$ (but not the rewards), possesses a set $\cA_{Chloe}(\{a_t\}_{t=1}^T)$ of suspected arms. Clearly, Chloe wishes $\cA_{Chloe}$ to be as small as possible, yet include $a(i^*)$ with high probability. We thus \emph{measure Chloe's missed--information} by the size of $\cA_{Chloe}$. Specifically, we define the following.
\begin{definition}[Security Constraint]\label{security}
We say that an algorithm is \emph{secure with equivocation $E$} if for any method Chloe uses to construct $\cA_{Chloe}(\{a_t\}_{t=1}^T)$ such that $a(i^*) \in \cA_{Chloe}$ with probability at least $1-e^{-O\left(\frac{T}{d}\right)}$, we have $\log|\cA_{Chloe}| \ge E$.
\end{definition}
Clearly, even without observing arm pulls, Chloe may set $\cA_{Chloe} = \cA$, which includes the best arm with probability $1$, yet Chloe's equivocation is $\log(K)$. Naively, this would be the equivocation one would wish to impose on Chloe. However, in a linear problem, when the dimension of the arms is $d$, hence all are spanned by some basis of $d$ vectors, we argue that a reasonable demand for security is to keep Chloe with a set of about $d$ candidates she cannot distinguish. The examples below further stress out that the true size variable is $d$. This is also the reason why $K$ is absent in most asymptotic expressions for the error, and the critical parameters are $T$ and $d$ (and some hardness measure). We thus argue that a sufficient equivocation result to ensure security is $\log(d)$.   
%However, based on, for example, \Cref{single round example} below, it is clear that a player can identify the best arm with exponentially low error probability by using a completely non--adaptive algorithm, which does not use the rewards to select future arm pulls and hence reveals no sensitive information to Chloe. Such an algorithm will not use all $K$ arms, only at most $O(d^2)$ of them (see below). This is done by using arms \emph{suitable for best estimation} -- a knowledge Chloe can gain offline on her own, based on the set $\cA$ alone. This is not surprising, as the problem is linear and the influential dimension is $d$ and not $K$. 
%%%%%%%%%%%%%%%%%%%%%%%%%%%%%%%%%%%%%%%%%%
\subsection{Example: A Single (non--adaptive) Round}\label{single round example}
Consider an algorithm spending all budget of $T$ pulls in a single round, and making a decision based on the result. Clearly, such an algorithm reveals nothing on the player's estimate. Still, to optimize estimation performance, the player will choose an appropriate set of arms, and may decide on a non-uniform distribution on the arm--pulls, e.g., by using the G--optimal design used in \cite{yang2022minimax}.

Let $\hat{p}(i)$ denote the estimate of $p(i)$ after this round. By \cite[equation (10)]{yang2022minimax}, we have
\begin{equation*}
    P\left(\hat{p}(1) < \hat{p}(i)\right) \leq \exp\left(-\frac{\Delta_i^2}{2\|a(1)-a(i)\|_{V^{-1}}^2}\right).
\end{equation*}
Hence, we obtain 
\begin{eqnarray}
    P_e &=& P\left(\cup_{i \ge 2} \hat{p}(1) < \hat{p}(i)\right)\nonumber
    \\
    & \leq &
    \exp\left(- \frac{T}{4 H_{2,\text{lin}} d} + \log K\right)\label{equation for example}
\end{eqnarray}
which is proved in \cite[\Cref{proof of equation for example}]{CohenGunlu2025arxiv}.

This suggests an error probability of $\exp\left(-\Omega\left(\frac{T}{H_{2,\text{lin}} d}\right)\right)$, rather than the optimal $\exp\left(-\Omega\left(\frac{T}{H_{2,\text{lin}} \log d}\right)\right)$, which might be significant in many applications where $d$ is very large. In terms of security, it is clear that Chloe only sees the arms used regardless of their rewards, she has no data to infer which of them is better. Hence, her equivocation is $O(\log(d))$ since at most $O(d^2)$ arms are used. Some intuition on how the $\exp\left(-\Omega\left(\frac{T}{H_{2,\text{lin}} \log d}\right)\right)$ exponent is achieved without revealing information on the best arm is given in \cite[\Cref{app intuition}]{CohenGunlu2025arxiv}.
%%%%%%%%%%%%%%%%%%%%%%%%%%%%%%%%%%%%%%%%%%
\subsection{Example: Estimating Each Entry of $\theta^*$}
Another straightforward solution that does not reveal anything to Chloe yet identifies $i^*$ with an error probability bounded by $\exp(\Omega(-\frac{T}{H_{2,\text{lin}} d}))$, is to estimate each entry of $\theta^*$ separately -- using linear combinations of the arms (coded arms) to directly estimate the entries. 

Specifically, as we assume that $\cA$ spans $\R^d$, one can construct linear combinations of the arms to achieve any unit vector $e_i \in \R^d$. Pulling each $e_i$ exactly $T'$ times will estimate the $i$--th entry of $\theta^*$ with an error probability bounded by $\exp(\Omega(-T'))$. However, here, again, one will have to take $T'=T/d$ to estimate all entries, resulting in the division by $d$ in the final exponent.
%%%%%%%%%%%%%%%%%%%%%%%%%%%%%%%%%%%%%%%%%%
\subsection{A Lower Bound with Uncoded Arms}
A natural question which arises from the above discussions is what is the minimal number of arm pulls that will allow Leah to identify $i^*$ with high probability, yet keep Chloe ignorant, with an equivocation at least $\log d$. 
\begin{lemma}\label{lower bound}
Assume an uncoded algorithm, i.e., Leah is restricted to use only original arms. Then, any algorithm with an equivocation at least $\log d$, results in an error exponent of $O\left(\frac{T\Delta_2^2}{d}\right)$ for identifying the best arm at Leah's side.
\end{lemma}
The proof is based on a genie--aided argument for separating the two--best arms. However, note that only when arms are well--separated, i.e., $\Delta_2^2 \approx \frac{2}{H_{2,\text{lin}}}$, the result is strong and suggests an $O\left(\frac{T}{H_{2,\text{lin}} d}\right)$ lower bound. When all arms are very close to the best one, i.e., $\Delta_2^2 \approx \frac{d}{H_{2,\text{lin}}}$, the lower bound on the exponent is $O\left(\frac{T}{H_{2,\text{lin}}}\right)$, which is much weaker. The proof is given in \cite[\Cref{proof of lower bound}]{CohenGunlu2025arxiv}.
%%%%%%%%%%%%%%%%%%%%%%%%%%%%%%%%%%%%%%%%%%
\section{A Secure Algorithm with an $\Omega\left(\frac{T}{H_{2,\text{lin}} \log^2 d}\right)$ Error Exponent}
\begin{theorem}\label{main theorem}
Consider the Stochastic Linear Best Arm Identification problem with $K$ arms of dimension $d$ and $T$ arm pulls. Assume the noise terms for each arm pull are Gaussian with variance 1. Then, \Cref{alg:secureBAI} satisfies the security constraint in \Cref{security}, achieving an equivocation of $\log(d)-1$ and an error exponent of $\Omega\left(\frac{T}{H_{2,\text{lin}} \log^2 d}\right)$.
\end{theorem}

%%%%%%%%%%%%%%%%
\begin{algorithm}
%\vspace{-0.2cm}
\caption{Secure Linear Best Arm Identification}\label{alg:secureBAI}
\renewcommand{\algorithmicrequire}{\textbf{Input:}}
\renewcommand{\algorithmicensure}{\textbf{Output:}}
\begin{algorithmic}[1]
\Require $T, \cA = \{a(1),\ldots, a(K)\} \subset \R^d, \text{Span}(\cA) = d=2^r$
\Ensure The only arm in $\cA_{\log d}$
\State $t_1 \gets 1$
\State $\cA_0 \gets \cA$
\State $d_0 \gets d$
\State $m \gets \frac{T - \min(K,\frac{d(d+1)}{2})-d+2 }{\log d}$\label{setting m} \Comment{Roughly \#pulls in round}
%%%%
\LeftComment{Round 1 - Uncoded}
\State Find a G--optimal design $\pi_1$ for $\cA_{0}$ 
\State $\cS \gets \text{Support}(\pi_1)$ \Comment{Bounded support}
\State $T_1(i) \gets \ceil{m\pi_1(i)}$
\State $T_1 \gets \sum_{i \in \cA_{0}}T_1(i)$
\State Choose arm $a(i)$ exactly $T_1(i)$ times
\State $V_1 \gets \sum_{i \in \cA_{0}}T_1(i) a(i) a(i)^T$ 
\State $\hat{\theta}_1 \gets V_1^{-1} \sum_{t=t_1}^{t_1+T_1-1} a_tX_t$ \Comment{Estimate $\theta^*$}
\For {$i \in \cA_{0}$} \Comment{Estimate all arms}
    \State $\hat{p_1}(i) \gets a(i)^T \hat{\theta}_1$ 
\EndFor     
\State $\cA_1 \gets \frac{d}{2} \text{ best arms in } \cA_{0}$ \Comment{Elimination}
\State $\cA_c \gets \frac{d}{2} \text{ arbitrary arms in } \cS \setminus \cA_1$ \Comment{Dummy arms}
\State $\cM_1 \gets \left\{ \{a(i)\} \right\}, i \in \cA_1 \cup \cA_c$ \Comment{A multiset}
\State $t_{2} \gets t_1 + T_1$
%%%%
\LeftComment{Rounds $2$ to $\log d$ - Coded rounds}
\For {$r=2$ to $\log d$}
    %\State $d_r \gets \text{dim}(\text{Span}(\cA_{r-1}))$
    %\State Find an orthonormal basis $B_r \in \R^{d \times d_r}$ for $\cA_{r-1}$
    \State Find a G--optimal design $\pi_r$ for $\cA_{r-1}$   
    \State $T_r(i) \gets \ceil{m\pi_r(i)}$
    \State $T_r \gets \sum_{i \in \cA_{r-1}}T_r(i)$
    \LeftComment{Play coded arms}
    \State $\cM_r \gets \text{union of pairs of random sets in } \cM_{r-1}$ 
    \NoNumber{\Comment{$\cM_r$ has $\frac{d}{2^{r-1}}$ subsets, $2^{r-1}$ arms in each}}
    \For {$i \in \cA_{r-1}$}
        \State $c(i) \gets \sum_{j \in \cM_r (s) \text{ s.t.\ } i \in \cM_r (s)} a(j)$
        \NoNumber{\Comment{Sum of all arms in the subset including $a(i)$}}
        \For {$t=\sum_{j \in \cA_{r-1}, j<i}T_r(j)+1$ to $\sum_{j \in \cA_{r-1}, j \leq i}T_r(j)$}
            \State $X^c_t \gets$ reward for $c(i)$ \Comment{Play arm $c(i)$}
            \State $X_t \gets \text{Decode}(a_t, \{X^c_{t'}\}_{t' \leq t})$\label{decoding line}
            \NoNumber{\Comment{Decode reward for $a(i)$, played at time $t$ 
 from the played $c(i)$ at time $t$ and past rewards}}
        \EndFor
         
    \EndFor
    \LeftComment{Estimate}
    \State $V_r \gets \sum_{i \in \cA_{r-1}}T_r(i) a(i) a(i)^T$ 
    \State $\hat{\theta}_r \gets V_r^{-1} \sum_{t=t_r}^{t_r+T_r-1} a_tX_t$
    \For {$i \in \cA_{r-1}$}
        \State $\hat{p_r}(i) \gets a(i)^T\hat{\theta}_r$ \Comment{Estimate expected rewards}
    \EndFor 
    \State $\cA_r \gets \frac{d}{2^r} \text{ best arms in } \cA_{r-1}$ \Comment{Eliminate arms}
    \State $t_{r+1} \gets t_r + T_r$
\EndFor 
\end{algorithmic}
\end{algorithm}

As mentioned above, the key idea behind \Cref{alg:secureBAI} is the use of coded arm pulls. Specifically, it follows the structure outlined in \cite{yang2022minimax}, with a total of $\log(d)$ rounds (we outline this structure in \cite[\Cref{app OD-LinBAI}]{CohenGunlu2025arxiv} for completeness), but from the second round onward, each arm pulled is a \emph{linear combination of a carefully selected set of arms}. This coding strategy helps in keeping the arms suspected to be the best secure from copycats, as each arm from the first round onward appears an equal number of times in the arm pulls. Simultaneously, it enables the user of the algorithm to decode the correct results.

The proof of \Cref{main theorem} follows the steps in \cite{yang2022minimax}, but it requires careful attention due to the coded arms. This coding requires, first and foremost, a decoding process, which must succeed. Then, the estimation process is different, as it uses the decoded arms with an increased level of noise. Moreover, the noise terms are correlated. These have to be accounted for. Finally, one has to prove that the equivocation constraint is achieved, and the number of arm pulls is at most $T$. Herein, we give four main lemmas that constitute these stages. 

Before the estimation step for $\theta^*$, which occurs at the end of each round, the algorithm incorporates a decoding process. This process extracts the results of the required arm pulls, which are necessary for the estimation, from the actual coded arm pulls. The first lemma below demonstrates that decoding is indeed possible. 
\begin{lemma}\label{decoding lemma}
The Decode procedure in \Cref{decoding line} of \Cref{alg:secureBAI} never fails. Specifically, at round $r$, one can decode $a(i)^T \theta^*$ using a sparse linear combination of $r$ values from $\{X^c_{t'}\}_{t' \leq t}$ with a noise term of variance $r$.
\end{lemma}
\begin{proof}   
First, note that the support of $\pi_r$ is at most $\frac{d(d+1)}{2}$ when $r=1$. Thus, at round 1, the algorithm plays with at most $\frac{d(d+1)}{2}$ uncoded arms. Then, from round 2 onward, the algorithm plays with coded arms, then decodes the rewards for the original arms before the estimation process.

The procedure $\text{Decode}(a_t, \{X^c_{t'}\}_{t' \leq t})$, called at round $2 \leq r \leq \log d$, is required to decode a new instance of $a_t^T\theta^*$ from the past results of arm plays $\{X^c_{t'}\}_{t' \leq t}$ (coded and uncoded). The arm $a_t$, for  $\sum_{j \in \cA_{r-1}, j<i}T_r(j)+1 \leq t \leq \sum_{j \in \cA_{r-1}, j \leq i}T_r(j)$, refers to a specific arm $i \in \cA_{r-1}$, one of $\frac{d}{2^{r-1}}$ arms. At time $t$, the decoder has a new measurement $X_t^c$, which is the reward for playing coded arm $c(i)$. This coded arm, in turn, is a linear combination of $2^{r-1}$ arms, one of which is the required $a(i)$. Let $s$ be the index of the subset in the multiset $\cM_r$ which contains $a(i)$. We have
\begin{equation*}
    X_t^c = c(i)^T \theta^*+ \eta_t
    = \sum_{j \in \cM_r(s)}a(j)^T \theta^* + \eta_t.
\end{equation*}
Note that $|\cM_r(s)| = 2^{r-1}$, and that $\cM_r(s)$ is a union of two subsets (from round $r-1$) of size $2^{r-2}$ each, only one of which includes $a(i)$. Denote these two subsets by $\cM_{r-1}(u)$ and $\cM_{r-1}(u_{-i})$. Note also that both $\sum_{l \in \cM_{r-1}(u)}a(l)$ and $\sum_{l \in \cM_{r-1}(u_{-i})}a(l)$ were played at round $r-1$. Denote the time instant when $\sum_{l \in \cM_{r-1}(u_{-i})}a(l)$ was played as $t_{r-1}$. We have
\begin{align*}
    X_t^c &= \sum_{j \in \cM_{r-1}(u_{-i})}a(j)^T \theta^* + \sum_{j \in \cM_{r-1}(u)}a(j)^T \theta^* + \eta_t
    \\
    & = X_{t_{r-1}}^c -\eta_{t_{r-1}} + \sum_{j \in \cM_{r-1}(u)}a(j)^T \theta^* + \eta_t.
\end{align*}
This process can be repeated recursively for $\cM_{r-1}(u)$. For a stopping condition, note that at round $2$, the required $a(i)$ was played in a linear combination with some other arm, which, itself, was played \emph{alone} at round $1$. Hence, we have
\begin{multline*}
    X_t^c = X_{t_{r-1}}^c-\eta_{t_{r-1}} + X_{t_{r-2}}^c-\eta_{t_{r-2}} \\+ \ldots + X_{t_{2}}^c-\eta_{t_{2}} + X_{t_{1}} -\eta_{t_1}+ a(i)^T \theta^* + \eta_t.
\end{multline*}
As a result, we obtain
\begin{multline}\label{decoding of arm a(i)}
   a(i)^T \theta^* + \eta_t - \eta_{t_{r-1}} - \eta_{t_{r-2}} - \ldots - \eta_{t_1}\\ =  X_t^c - X_{t_{r-1}}^c - X_{t_{r-2}}^c - \ldots - X_{t_{2}}^c - X_{t_{1}}.
\end{multline}
That is, a new arm pull $a(i)^T \theta^* + \eta_t^c$ at time $t$ can be decoded from $X_t^c$ and the pulls at rounds $r-1$ to $1$. Note that the noise term $\eta_t^c = \eta_t - \eta_{t_{r-1}} - \eta_{t_{r-2}} - \ldots - \eta_{t_1}$ has variance $r$.
%\textcolor{red}{mention we have different previous results to choose from - we don't have to use previous arms again and again}
\end{proof}
%%%%%%%%%%%%%
The second lemma shows that while the estimation error does increase, it only does so by at most a factor of $2\log(d)$. This results in an $O(\log^2(d))$ term in the denominator of the exponent, higher than the $\log(d)$ term found in the un--secure version \cite{yang2022minimax} but much lower than the linear $O(d)$ term when using trivial, uncoded arm pulls.
\begin{lemma}\label{error lemma}
The error exponent in each of the $\log(d)$ rounds of \Cref{alg:secureBAI} is at least $\Omega\left(\frac{T}{H_{2,\text{lin}} \log^2 d}\right)$.
\end{lemma}
The proof of \Cref{error lemma} is technical and appears in \cite[\Cref{app proof of lemma}]{CohenGunlu2025arxiv}. We briefly summarize the main arguments. One can notice from \eqref{decoding of arm a(i)} that the effective noise added to a decoded arm pull is, in fact, the sum of $r$ independent original noise variables. Each has variance $1$, hence the total variance is $r$, which is at most $\log(d)$. However, since different decoded arm pulls may use \emph{the same coded arm pulls from previous rounds}, effective noise variables for different decoded arms can be correlated. Fortunately, since there are repeated pulls at previous rounds as well, the number of correlated noise variables is not too high, keeping the effective variance in the tail-bound computation at $O(\log(d))$. This eventually increases the variance term in the error bound from $1$ to $O(\log(d))$, resulting in the exponent $\Omega\left(\frac{T}{H_{2,\text{lin}} \log^2 d}\right)$ instead of $\Omega\left(\frac{T}{H_{2,\text{lin}} \log d}\right)$.

%%%%%%%%%%%%%
%Now, we turn to the equivocation.
\begin{lemma}\label{security lemma}
The equivocation achieved is $\log(d)-1$.
\end{lemma}
\begin{proof}
At the first round, the player plays at most $O(d^2)$ arms. At the second round $d/2$ arms are used. From this point on, all $d/2$ arms repeat in the coded arms with the exact same frequency. This is since, in fact, all of these $d/2$ arms are coded together in each round (either in pairs, quarters, or more). Thus, Chloe is left with $d/2$ arms which appear equally likely. This result in an equivocation of $\log(d)-1$.
\end{proof}
%%%%%%%%%%%%%%%%%
%Finally, the lemma below asserts that \Cref{alg:secureBAI} does not exceed the limit of $T$ arm pulls.
\begin{lemma}
\Cref{alg:secureBAI} uses at most $T$ arm pulls.
\end{lemma}
\begin{proof}
Note that $m$ is set according to \cite[Equation (1)]{yang2022minimax}, and throughout the execution of the algorithm, the \emph{number} of arm pulls is equal to that in \cite[Algorithm 1]{yang2022minimax}. Hence, by \cite[Lemma 1]{yang2022minimax}, the algorithm terminates after at most $T$ arm pulls.
\end{proof}
%%%%%%%%%%%%%%%%%%%%%%%%%%%%%%%%%%%%%%%%%%
\section{Conclusion}
This paper introduces a novel approach to best arm identification problem that effectively addresses the challenge of information disclosure while learning. By employing coded arm pulls, the proposed algorithm allows us to explore potentially best arms without revealing them to outside observers. 

The algorithm demonstrates a significant improvement in performance compared to a naive, uncoded approach, achieving an error exponent of $\Omega\left(\frac{T}{\log^2(d)}\right)$ without relying on cryptographic methods, yet enforcing an $O(\log(d))$ equivocation. This strikes a balance between the need for effective  learning and the importance of maintaining confidentiality. The results underscore the potential for coded arms as a practical tool, enabling decision--makers to navigate competitive environments more securely while optimizing rewards.

\clearpage
%%%%%%%%%%%%%%%%%%%%%%%%%%%%%%%%%%%%%%%%%%
\newpage
\bibliographystyle{IEEEtran}
\bibliography{bandit_refs,ISF2025_refs}
%%%%%%%%%%%%%%%%%%%%%%%%%%%%%
\newpage
%%%%%%%%%%%%%%%%%%%%%%%%%%%%%%%%%%%%%%%%%%
\appendix
\subsection{Proof of Equation \ref{equation for example}}\label{proof of equation for example}
We have
\begin{eqnarray}
    P_e &=& P\left(\cup_{i \ge 2} \hat{p}(1) < \hat{p}(i)\right) 
    \\
    &\leq& \sum_{i \ge 2}\exp\left(-\frac{\Delta_i^2}{2\|a(1)-a(i)\|_{V^{-1}}^2}\right)
    \\
    &\leq& K \exp\left(- \min_{i \ge 2}\frac{\Delta_i^2}{2\|a(1)-a(i)\|_{V^{-1}}^2}\right)
    \\
    & \leq& K \exp\left(- \min_{i \ge 2}\frac{T\Delta_i^2}{8d}\right)\label{using Tan's result from Lemma 2}
    \\
    & \leq& K \exp\left(-\frac{T\Delta_2^2}{8d}\right)
    \\
    & \leq &
    \exp\left(- \frac{T}{4 H_{2,\text{lin}} d} + \log K\right)\label{using H2 bound}
\end{eqnarray}
where \eqref{using Tan's result from Lemma 2} results from \cite[Eq.~(13)]{yang2022minimax} and \eqref{using H2 bound} follows since $\frac{1}{H_{2,\text{lin}}} = \min_{i \ge 2}\left\{\frac{\Delta_i^2}{i}\right\} \leq \frac{\Delta_2^2}{2}$.
%%%%%%%%%%%%%%%%%%%%%%%%%%%%%%%%%%%%%%
\subsection{Proof of \Cref{lower bound}}\label{proof of lower bound}
We prove the lemma by suggesting a strategy for Chloe that will enforce poor performance on Leah. Chloe's strategy for analyzing the arm pulls can be very simple, yet keep Leah's performance poor unless she reveals information on the best arm. Specifically, Chloe simply sets a threshold $T'$. Each arm which is pulled more than $T'$ times (anywhere throughout the $T$ pulls) is considered suspicious and added to $\cA_{Chloe}$.

We now argue that if Leah wishes to identify the best arm with high probability, that arm has to be pulled enough times (i.e., at least $\tilde{T}$ for large enough $\tilde{T}$), moreover, to keep $\cA_{Chloe}$ large enough, Leah has to pull other arms at least that much, hence, for example, to enforce an equivocation $\log d$ Leah has to pull $d-1$ additional arms at least $\tilde{T}$ times each, resulting in poor performance. 

Assume Leah is genie--aided, and knows completely all arms' distributions except the two best arms. To have a vanishing probability of error, Leah must distinguish between the two best arms. As the rewards of the two best arms are, by definition, $\Delta_2$ apart, Leah has to estimate each arm's reward with a precision at least $\frac{\Delta_2}{2}$.

Each arm pull results in a measurement of the reward, with an error term which is Gaussian with variance $1$. Thus, pulling arm $i$ $\tilde{T}$ times results in an error probability bounded by
\begin{align*}
    P\left(\left|\frac{1}{\tilde{T}}\sum_{t=1}^{\tilde{T}}X_t - p(i)\right| > \epsilon\right) &=  2P\left(Z > \sqrt{\tilde{T}}\epsilon\right)
    \\
    & \ge \frac{2}{\sqrt{2\pi}}\frac{\sqrt{\tilde{T}}\epsilon}{\tilde{T} \epsilon^2 +1} \exp\left(-\frac{\tilde{T} \epsilon^2}{2}\right)
\end{align*}
when $Z$ is standard normal. Taking $\epsilon = \frac{\Delta_2}{2}$ results in an error exponent which is \emph{at most} $\frac{\tilde{T}\Delta_2^2}{8}$. That is, for an error exponent of $O\left(\tilde{T}\Delta_2^2\right)$, Leah has to pull at least $2$ arms, at least $\tilde{T}$ times each. This is a computation Chloe can easily do, hence set her threshold accordingly. In fact, Chloe can tune her threshold to enforce a lower bound on Leah's error probability. Now, to further ensure Chloe's set is large, Leah has to pull enough other arms at least as much. For an equivocation $\log d$, Leah has to pull $d$ arms in total. This results in $T = d \tilde{T}$, hence an error exponent of  $O\left(\frac{T\Delta_2^2}{d}\right)$.

Note, again, that this is assuming each arm pull is separate and includes only original arms -- without coding.
%%%%%%%%%%%%%%%%%%%%%%%%%%%%%%%%%%%%%%%%%%
\subsection{Proof of \Cref{error lemma}}\label{app proof of lemma}
At round $r$ of \Cref{alg:secureBAI}, the algorithm pulls $T_r = \sum_{i \in \cA_{r-1}}T_r(i)$ coded arms. Each suspected arm $i \in \cA_{r-1}$ is included in a coded arm $c(i)$, which is pulled $T_r(i)$ times. By \Cref{decoding lemma}, the result for arm $i$ can be decoded from the arm pull for $c(i)$. Specifically, by \eqref{decoding of arm a(i)},
\begin{equation}\label{decoded rewards}
   a(i)^T \theta^* + \eta^c_t  = X_t^c - X_{t_{r-1}}^c - X_{t_{r-2}}^c - \ldots - X_{t_{2}}^c - X_{t_{1}}
\end{equation}
with $\eta_t^c = \eta_t - \eta_{t_{r-1}} - \eta_{t_{r-2}} - \ldots - \eta_{t_1}$. However, the noise terms $\{\eta_t^c\}_{\sum_{1}^{r-1}T_r <  t \leq \sum_{1}^{r}T_r}$ are correlated. This is since while $\eta_t$ in $\eta_t^c$ is a new noise term from the current (t) arm pull, $\eta_{t_{r-1}}$ to $\eta_{t_1}$ are noise terms from previous arm pulls, which might appear in the decoding process of different coded arm pulls. That is, while decoding $a(i)^T \theta^*$ at a different time $t'$ (at the same round), or some $a(j)^T \theta^*$ for $j \in \cA_{r-1}$ but $j \ne i$, some noise terms may be the same, as the decoding process used a similar coded arm from a previous round. Such overlapping noise terms clearly create a correlation between the noise terms $\{\eta_t^c\}_{\sum_{1}^{r-1}T_r <  t \leq \sum_{1}^{r}T_r}$ of the decoded arms, hence the estimation error may deteriorate.

To bound the estimation error with correlated noise terms, we revisit the result in \eqref{upper bound for estimation}. Without regularization, the estimated $\theta^*$ at round $r$ is given by  \begin{equation}
\hat{\theta}_r = V_r^{-1} \sum_{t=1}^{T_r} (a_t^T \theta^* + \eta_t^c) a_t
\end{equation}
where $V_r = \sum_{t=1}^{T_r} a_t a_t^T$, and, with slight abuse of notation, $\{a_t\}_{t=1}^{T_r}$ are the arms played in round $r$ (the original arms after decoding -- the arms in $\cA_{r-1}$, each played $T_r(i)$ times), and $\{a_t^T \theta^*\}_{t=1}^{T_r}$ are the decoded rewards, i.e., those given in \eqref{decoded rewards}.

Note that for any $w \in \R^d$
\begin{align*}
\langle \hat{\theta}_r - \theta^*, w\rangle &= \langle V_r^{-1} \sum_{t=1}^{T_r} (a_t^T \theta^* + \eta_t^c) a_t - \theta^*, w\rangle
\\
& = \langle V_r^{-1} \sum_{t=1}^{T_r} a_t a_t^T \theta^* + V_r^{-1} \sum_{t=1}^{T_r}a_t\eta_t^c - \theta^*, w\rangle
\\
& = \langle  V_r^{-1} \sum_{t=1}^{T_r}a_t\eta_t^c, w\rangle
\\
& = \sum_{t=1}^{T_r} \langle  V_r^{-1}a_t, w\rangle \eta_t^c.
\end{align*}
Thus, with Gaussian noise, $\langle \hat{\theta}_r - \theta^*, w\rangle$ is a sum of $T_r$ Gaussian variables, which is Gaussian. Since all noise terms have zero mean, $\text{E}\langle \hat{\theta}_r - \theta^*, w\rangle =0$. However, the crux of the matter is that the variables $\{\eta_t^c\}_{t=1}^{T_r}$ are correlated. Thus, a direct tail bound for i.i.d. variables cannot be applied to achieve a result similar to \eqref{upper bound for estimation}. 

Yet, denoting $\bx^T A \bx = Q_A(\bx)$ for a covariance matrix A, and applying it for the covariance matrix of the noise terms $\{\eta_t^c\}_{t=1}^{T_r}$, we have 
\begin{align*}
\text{Var}\{\langle \hat{\theta}_r - \theta^*, w\rangle\} &= \sum_{t=1}^{T_r}\sum_{s=1}^{T_r} \langle  V_r^{-1}a_t, w\rangle \text{E}\{\eta_t^c \eta_s^c\}\langle  V_r^{-1}a_s, w\rangle
\\
& = Q_{\Lambda_{\eta}}\left( \{\langle  V_r^{-1}a_t, w\rangle\}_{t=1}^{T_r} \right)
\\
& \leq ||\Lambda_{\eta}|| \sum_{t=1}^{T_r} \langle  V_r^{-1}a_t, w\rangle^2.
\end{align*}
On the other hand, $Q_{\Lambda_{\eta}}(\textbf{1}) \leq ||\Lambda_{\eta}|| T_r$, when $\textbf{1}$ is the all--ones vector of length $T_r$. If $\Lambda_{\eta}$ has constant row--sum, then $\textbf{1}$ is aligned with $\Lambda_{\eta}$'s dominant eigenvector, and we have $Q_{\Lambda_{\eta}}(\textbf{1}) = ||\Lambda_{\eta}|| T_r$. This is indeed the case, as subsets are united at random, uniformly for all rows. Specifically, note that each row in $\Lambda_{\eta}$ is the correlation between the noise terms in the decoded arms. The highest correlations are in the last round, round $\log(d)$. In this round, the variance of each variable is $\log(d)$, creating a diagonal of $\log(d)$ in $\Lambda_{\eta}$, and, moreover, to account for each non-diagonal elements in each row we note the following: There are $\frac{T}{\log(d)}$ noise terms $\eta_t^c$, each is a sum of $\log(d)$ variables. The first is an independent (new) variable of variance $1$, while the rest are variables of previous measurements (of previous rounds). Thus $\frac{T}{\log(d)} (\log(d) -1 ) = T(1-\frac{1}{\log(d)})$ terms are noise variables from previous rounds. On the other hand, the number of previous measurements is $T-\frac{T}{\log(d)} = T(1-\frac{1}{\log(d)})$. Hence, each of the noise terms composing $\eta_t^c$ which is not new appears exactly once before, contributing for a total of $\log(d)-1$ to the row sum. Together with the diagonal element, a row sum is $2\log(d)-1$. This also results in $Q_{\Lambda_{\eta}}(\textbf{1}) \leq 2 \log(d) T_r$. Hence, 
\begin{align*}
\text{Var}\{\langle \hat{\theta}_r - \theta^*, w\rangle\} 
& \leq ||\Lambda_{\eta}|| \sum_{t=1}^{T_r} \langle  V_r^{-1}a_t, w\rangle^2
\\
& = \frac{Q_{\Lambda_{\eta}}(\textbf{1})}{T_r}\sum_{t=1}^{T_r} \langle  V_r^{-1}a_t, w\rangle^2
\\
& \leq 2\log(d) \sum_{t=1}^{T_r} \langle  V_r^{-1}a_t, w\rangle^2.
\end{align*}
Applying the exponential bound $P(X>\epsilon) \leq e^{-\frac{\epsilon^2}{2\sigma^2}}$ to $\langle \hat{\theta}_r - \theta^*, w\rangle$ with the above variance bound, noting that $\sum_{t=1}^{T_r} \langle  V_r^{-1}a_t, w\rangle^2 = \|w\|^2_{V_r^{-1}}$, results in 
\begin{equation}\label{updated bound with log(d)}
P\left(\langle \hat{\theta}-\theta^*, w \rangle \ge \sqrt{4 \log(d) \|w\|^2_{V_r^{-1}}\log\left(\frac{1}{\delta}\right)}  \right) \leq \delta. 
\end{equation}
The rest of the proof follows the chain of inequalities \cite[(9)-(14)]{yang2022minimax} almost directly, with the added $2\log(d)$ term for the increased noise. That is, remember that $\hat{p}_r(i)$ denotes the estimate of $p_r(i)$ at round $r$. We have
\begin{align*}
P(\hat{p}(1) < \hat{p}(i)) &= P\left( \langle \hat{\theta}_r-\theta^* , a_r(1)-a_r(i)\rangle < -\Delta_i \right)
\\
& \leq \exp\left( - \frac{\Delta_i^2}{4 \log(d) || a_r(1)-a_r(i)||_{V_r^{-1}}}\right).
\end{align*}
Compared to the $\Omega\left(\frac{T}{H_{2,\text{lin}} \log d}\right)$ exponent in \cite{yang2022minimax}, this now results in an $\Omega\left(\frac{T}{H_{2,\text{lin}} \log^2 d}\right)$ exponent.
%%%%%%%%%%%%%%%%%
\subsection{OD-LinBAI}\label{app OD-LinBAI}
For completeness, we give here the original OD-LinBAI algorithm of \cite{yang2022minimax} with a slightly modified presentation to be comparable to the representation of \Cref{alg:secureBAI}. Note that we use here the reduced--dimension version, which is not included in \Cref{alg:secureBAI}.
\begin{algorithm}
\caption{OD-LinBAI}\label{alg:BAI}
\renewcommand{\algorithmicrequire}{\textbf{Input:}}
\renewcommand{\algorithmicensure}{\textbf{Output:}}
\begin{algorithmic}[1]
\Require $T, \cA = \{a(1),\ldots, a(K)\} \subset \R^d, \text{Span}(\cA) = d=2^r$
\Ensure The only arm in $\cA_{\log d}$
\State $t_1 \gets 1$
\State $\cA_0 \gets \cA$
\State $d_0 \gets d$
\State Calculate $m$
\For {$r=1$ to $\log d$}
    \State $d_r \gets \text{dim}(\text{Span}(\cA_{r-1}))$
    \If {$d_r = d$} \Comment{Only for $r=1$}
        \State $B_r = I_{d \times d}$ 
    \Else\Comment{$d_r < d$ for $r>1$}
        \State Find an orthonormal basis $B_r \in \R^{d \times d_r}$ for $\cA_{r-1}$
    \EndIf
    \State Find a G--optimal design $\pi_r$ for $B_r^T \cA_{r-1}$ \label{arm-wise}
    \State $T_r(i) \gets \ceil{m\pi_r(i)}$
    \State $T_r \gets \sum_{i \in \cA_{r-1}}T_r(i)$
    \State Choose arm $B_r^Ta(i)$ $T_r(i)$ times \label{playing}\Comment{Play}
    \State $V_r \gets B_r^T\left(\sum_{i \in \cA_{r-1}}T_r(i) a(i) a(i)^T\right) B_r$ 
    \State $\hat{\theta}_r \gets V_r^{-1} \sum_{t=t_r}^{t_r+T_r-1}B_r^T a(A_t)X_t$\label{estimation}
    \For {$i \in \cA_{r-1}$}
        \State $\hat{p_r}(i) \gets \hat{\theta}_r^T B_r^T a(i)$ \Comment{Estimate expected rewards}
    \EndFor 
    \State $\cA_r \gets \frac{d}{2^r} \text{ best arms in } \cA_{r-1}$ \Comment{Eliminate arms}
    \State $t_{r+1} \gets t_r + T_r$
\EndFor 
\end{algorithmic}
\end{algorithm}

Note that since the columns of $B_r$ span the subspace of $\cA_{r-1}$, each $a(i) \in \cA_{r-1}$ can be represented as a linear combination of these columns, hence $a(i) = B_r B_r^T a(i)$ and we have $a(i)^T \theta^* = (B_r B_r^T a(i))^T \theta^*$. Moreover, $a(i)^T \theta^* = (B_r^T a(i))^T(B_r^T\theta^*)$. Hence, one can estimate the reduced--dimension $B_r^T\theta^*$, using the reduced--dimension arms $B_r^T a(i)$. In practice, one keeps the arms as $B_r^T a(i)$ (for ease of estimation and computing the G-optimal design) but plays with full--dimension arms $B_r B_r^T a(i)$. The estimation process estimates the reduced--dimension vector. For example:
At line \ref{arm-wise} we assume operating on $\cA_{r-1}$ arm--wise. At line \ref{playing}, the arm in the set maintained is $B_r^Ta(i)$ but the player plays with the $d$--dimensional $B_rB_r^Ta(i)$. At line \ref{estimation} we actually estimate the reduced--dimension $\theta^*$, that is $B_r^T\theta^*$.
%%%%%
\subsection{Intuition on the Error Exponent and Gained Information}\label{app intuition}
One may wonder how is the division by $d$ in the exponent of the single round algorithm is evaded in the multi--round elimination process. Indeed, in a multi--round process, the best arm should not necessarily have the \emph{highest estimated reward}, but only has to survive the elimination process towards the next round. Hence, it does not have to win over $K$ arms in each round, but only not to loose to more than $\frac{d}{2^r}$ arms in the $r$--th round. Then, as the rounds evolve, the better arms are tested more times (compared to the first round), and their estimation improves. This is the essence of the exploration versus exploitation trade--off. In fact, \cite{yang2022minimax} uses $\log d$ rounds, with about $\frac{T}{\log d}$ arm pulls in each. This is the reason for the division in $\log d$ in the final exponent (the union over the rounds is negligible).\footnote{A back of the envelope calculation shows that towards, e.g., the second round, the probability of the best arm loosing to $d/2$ arms is roughly ${K \choose d/2} \left(e^{-\frac{T}{H_{2,\text{lin}}d}}\right)^{\frac{d}{2}} = \text{Poly}(K,d)e^{-\frac{T}{2H_{2,\text{lin}}}}$, which results in $\exp(\Omega(-T/H_{2,\text{lin}}))$, and not the single--round exponent which is divided by $d$. The same follows for the reminder of the rounds. However, the complete algorithm includes $\log d$ rounds, hence using $m \approx T/\log d$ pulls per round gives the final $\exp(\Omega(-\frac{T}{H_{2,\text{lin}}\log d}))$ exponent.}

%%%%%%%%%%%%%%%%%%%%%%%%%%%%%%%%%%%%%%%%%%
\end{document}